\documentclass{article}

\PassOptionsToPackage{numbers,compress}{natbib}
\usepackage[preprint]{neurips_2026}

\usepackage[utf8]{inputenc}
\usepackage[T1]{fontenc}
\usepackage{hyperref}
\hypersetup{hidelinks,hypertexnames=false}
\usepackage{url}
\usepackage{booktabs}
\usepackage{microtype}
\usepackage{xcolor}
\usepackage{amsmath,amssymb,amsfonts,amsthm,bm}
\usepackage{algorithm}
\usepackage{algpseudocode}
\usepackage{enumitem}
\usepackage{tikz}
\usetikzlibrary{arrows.meta,positioning,fit,calc,shapes.geometric}

\newtheorem{definition}{Definition}
\newtheorem{proposition}{Proposition}

\newcommand{\E}{\mathbb{E}}
\newcommand{\Prob}{\mathbb{P}}
\newcommand{\Wset}{\mathcal{W}}
\newcommand{\Ucal}{\mathcal{U}}

\title{SafeCommit: Certifying When Memory-Grounded Agents May Safely Act}

\author{
  Mayur Akewar\thanks{Mayur Akewar and Ravi Ranjan contributed equally to this work.} \\
  Florida International University \\
  \texttt{makew001@fiu.edu}
  \And
  Ravi Ranjan\footnotemark[1] \\
  Florida International University \\
  \texttt{rkuma031@fiu.edu}
}

\begin{document}
\maketitle

\begin{abstract}
Long-horizon agents increasingly use persistent memory and tools to take actions with external side effects. A central failure mode is premature commitment: an agent acts before resolving whether its memory grounding is stale, conflicting, incomplete, or corrupted. We formalize this problem as \emph{safe commitment under memory uncertainty} and introduce \textsc{SafeCommit}, a risk controlled layer between agent reasoning and external execution. The layer constructs a calibrated set of plausible latent worlds from memory, observations, tool outputs, provenance, and policy constraints. It permits a side-effectful action only when a conformal action certificate shows that the action is safe in every retained world. Otherwise, it selects a low-side-effect probe that targets the worlds blocking certification, or returns a conservative fallback. Under calibrated world coverage, the probability of an unsafe certified commit is at most the target level $\alpha$; with imperfect world proposal, the bound separates calibration and representation error. A dependency-free controlled simulator illustrates the safety-utility tradeoff and reproduces all reported results with one command. The goal is to offer a concrete approach for deciding not only \emph{what} an agent should do, but \emph{when the available evidence is sufficient to safely do it}. The source code and experimental artifacts are publicly available at \url{https://github.com/akewarmayur/SafeCommit}.
\end{abstract}

\section{Introduction}
\label{sec:intro}

Large language model (LLM) agents are moving from single-turn assistants toward systems that remember, reason, call tools, and act over extended interactions. ReAct established the value of interleaving reasoning and tool use \citep{yao2023react}, while Reflexion and newer memory systems use persistent feedback and organized memory to improve future behavior \citep{shinn2023reflexion,xu2025amem}. Benchmarks such as OSWorld and Mem2ActBench further show that long-horizon action depends on maintaining context across tools and changing environments \citep{xie2024osworld,shen2026mem2actbench}. Memory is therefore becoming part of an agent's control state rather than a passive store.

This shift creates a failure mode that stronger retrieval alone does not solve. A memory can be relevant and still be unsafe to act on: a file path may have been repurposed, a recipient may have changed, an earlier permission may have expired, or an injected record may imitate trusted history. Memory poisoning work demonstrates that malicious records can redirect downstream agent behavior \citep{chen2024agentpoison,dong2025minja}. In these cases, the agent may produce a plausible rationale while several safety-relevant interpretations of the current evidence remain unresolved.

Existing safeguards address adjacent but different questions. Clarification methods request missing information \citep{kobalczyk2025activetaskdisambiguation,suri2025structureduncertainty}; uncertainty-aware methods may abstain when confidence is low \citep{bonagiri2025checkyourself,zhang2026agenticuq}; access-control systems restrict which actions are permitted \citep{abaev2026agentguardian}; and effect-aware substrates expose or contain side effects through staging and sandboxing \citep{ruan2024toolemu,zhong2026yolofs}. These mechanisms are valuable, but none by itself supplies a decision rule for the following question: \emph{given memory, current observations, tool outputs, and policy constraints, is a proposed action safe under every still-plausible interpretation?}

We study this question as \emph{safe commitment under memory uncertainty}. At a decision step, the agent retains a calibrated set of plausible latent worlds. A world represents one safety-relevant interpretation, such as whether a record is current, a permission remains valid, or a retrieved memory is poisoned. A candidate action receives a certificate only if it is safe in all retained worlds. If no action is certifiable, the controller chooses a low-side-effect probefor example, a metadata read, permission check, staged diff, simulation, or clarification-that is expected to remove worlds blocking certification. If uncertainty cannot be resolved within budget, the controller defers, escalates, or abstains.

This perspective differs from assigning a scalar confidence to an action. The certificate is set-valued and counterfactual: it asks whether any retained world makes the action unsafe. It also connects uncertainty estimation to evidence acquisition because probes are valued by how much they shrink the \emph{uncertified region}, rather than by generic information gain.

\vspace{-0.2cm}
\paragraph{Positioning.}
Table~\ref{tab:positioning} summarizes the distinction. Memory and retrieval methods improve the evidence presented to an agent, but normally produce one working context. Selective-action methods decide whether a scalar confidence is high enough to act, but do not retain explicit safety-relevant alternatives. Access control checks whether an operation is permitted under a policy, which is necessary but does not resolve whether the underlying state or intent is stale. Effect-mediation systems make actions inspectable or reversible, but still require a rule for releasing the final effect. \textsc{SafeCommit} is designed as the decision layer that consumes signals from all four.

\begin{table}[t]
\centering
\small
\caption{Conceptual positioning. ``Worlds'' means explicit alternative safety-relevant states; ``targeted probe'' means evidence acquisition optimized for action certification.}
\label{tab:positioning}
\begin{tabular}{lcccc}
\toprule
\textbf{Approach} & \textbf{Main output} & \textbf{Worlds} & \textbf{Targeted probe} & \textbf{Risk target} \\
\midrule
Memory/retrieval & context &  &  &  \\
Uncertainty/quit & act or abstain &  & sometimes & scalar \\
Access control & allow or deny &  & policy check & policy \\
Effect mediation & staged effect &  & effect inspection & substrate \\
\textsc{SafeCommit} & commit/probe/fallback & \checkmark & \checkmark & $\alpha$ \\
\bottomrule
\end{tabular}
\end{table}

The distinction is especially important when an action is both permitted and well supported by one memory, yet unsafe in another plausible world. For example, an agent may be authorized to delete files in a workspace, but an old memory may point to a directory that has since become shared. Access control permits the deletion, retrieval surfaces the old instruction, and a confidence score may favor it. The remaining question is whether current evidence rules out the repurposed-directory world. \textsc{SafeCommit} makes that unresolved alternative explicit and blocks release until a metadata or staged-effect probe removes it.

Our contributions are:
\begin{enumerate}[leftmargin=*,itemsep=1pt,topsep=2pt]
    \item We formalize safe commitment for memory-grounded, side-effectful agents using calibrated plausible-world sets and action-level safety semantics.
    \item We introduce \textsc{SafeCommit}, a commitprobefallback controller with conformal action certificates, exact safety-equivalence compression, and an explicit decomposition of calibration and representation error.
    \item We provide a compact, dependency-free simulator and one-command experiment that isolates the mechanism across stale, conflicting, poisoned, and authorization-shifted memory. The experiment is intentionally a proof of concept, not a claim of full deployed-agent validation.
\end{enumerate}

\section{Problem Formulation}
\label{sec:formulation}

\subsection{Evidence, actions, and latent worlds}
At decision step $t$, the agent has local evidence
\begin{equation}
    x_t=(m_t,o_t,z_t,c_t),
\end{equation}
where $m_t$ is retrieved memory, $o_t$ is the current observation, $z_t$ contains recent tool outputs, and $c_t$ contains policy, authorization, or execution constraints. The history $h_t$ also includes prior probes and their outcomes. Let $\mathcal{A}_t$ be candidate side-effectful actions and $\mathcal{P}_t$ low-side-effect probes. Terminal decisions are
\begin{equation}
\mathcal{D}_t=\{\textsc{commit}(a):a\in\mathcal{A}_t\}\cup
\{\textsc{defer},\textsc{escalate},\textsc{abstain}\}.
\end{equation}
The underlying agent may propose actions and probes using an LLM. The controller treats these proposals as candidates rather than as evidence that they are safe.

Because the evidence may admit several interpretations, let $\Omega_t$ be a latent world space and $\omega_t^\star\in\Omega_t$ the true world. A constructor $C_\alpha$ maps the history to a plausible-world set
\begin{equation}
    \Wset_t=C_\alpha(h_t)\subseteq\Omega_t,
    \qquad \Prob(\omega_t^\star\in\Wset_t)\ge 1-\alpha.
    \label{eq:coverage}
\end{equation}
In practice, the constructor operates on a finite decision-relevant support $\widehat\Omega_t(h_t)$ or symbolic cells. Each candidate world receives a nonconformity score $s_t(\omega)$ based on provenance, age, consistency with current observations, tool agreement, and policy compatibility. With a split-conformal threshold $\kappa_\alpha$ estimated on held-out calibration cases,
\begin{equation}
\Wset_t=\{\omega\in\widehat\Omega_t(h_t):
\omega\text{ is hard-evidence consistent and }s_t(\omega)\le\kappa_\alpha\}.
\label{eq:worldset}
\end{equation}
This construction does not require enumerating every latent fact; it must preserve distinctions that can change action safety.

\subsection{Worked example}
Consider an agent asked to remove temporary outputs and send a completion note. Retrieved memory says that \texttt{/work/run/latest} is disposable and that \texttt{ops@example.org} is the approved recipient. Current evidence is incomplete: a directory listing shows recent files but not whether the path is now a symbolic link, and the recipient record has no recent provenance stamp. A single-state agent may infer that deletion and sending are both appropriate.

A decision-relevant support could instead contain three worlds. In $\omega_1$, the path remains temporary and the recipient is valid; both actions are safe. In $\omega_2$, the path was repurposed to shared results; deletion is unsafe although sending is safe. In $\omega_3$, the recipient memory was injected or superseded; sending is unsafe although a non-destructive file inspection is safe. If all three worlds survive calibration, neither deletion nor sending has an empty uncertified region.

The controller can now choose probes for the blocked decisions. A no-follow symbolic-link and ownership check distinguishes $\omega_1$ from $\omega_2$, while a directory or authorization lookup distinguishes $\omega_1$ from $\omega_3$. A generic question such as ``are you sure?'' may not resolve either fact. After the probes, the agent may certify the send action, replace deletion with a reversible staged move, or escalate if the recipient remains uncertain. The example illustrates why the framework represents only facts that can change action safety and why probes are tied to the action certificate rather than to broad uncertainty reduction.

\subsection{Action certificates}
Each world induces a set of safely admissible actions $\Gamma_t(\omega)\subseteq\mathcal{A}_t$. For action $a$, define the uncertified region
\begin{equation}
    \Ucal_t(a)=\{\omega\in\Wset_t:a\notin\Gamma_t(\omega)\}.
    \label{eq:uncertified}
\end{equation}

\begin{definition}[Conformal action certificate]
An action $a\in\mathcal{A}_t$ is $\alpha$-certified when $\Ucal_t(a)=\emptyset$.
\end{definition}
A plausible action may be safe in one likely world. A certified action remains safe in every world retained at the target coverage level.

\begin{proposition}[Unsafe-commit control]
\label{prop:risk}
Suppose the set at the stopping time $\tau$ satisfies Eq.~\eqref{eq:coverage}, and the controller commits only to actions with $\Ucal_\tau(a)=\emptyset$. Then
\begin{equation}
    \Prob\!\left(a_\tau\notin\Gamma_\tau(\omega_\tau^\star)\right)\le\alpha.
    \label{eq:risk}
\end{equation}
\end{proposition}
\begin{proof}
If $\omega_\tau^\star\in\Wset_\tau$ and the uncertified region is empty, the committed action is safe in the true world. Therefore an unsafe commit is possible only when $\omega_\tau^\star\notin\Wset_\tau$, whose probability is at most $\alpha$.
\end{proof}

A finite proposal mechanism can miss the true world before calibration. Let $\beta=\Prob(\omega_t^\star\notin\widehat\Omega_t(h_t))$. If conditional conformal coverage is at least $1-\alpha$ whenever the true world is proposed, then a union argument gives the practical bound
\begin{equation}
    \Prob(\text{unsafe commit})\le \beta+(1-\beta)\alpha\le\alpha+\beta.
    \label{eq:representation}
\end{equation}
This separates miscalibration from representation failure instead of hiding both behind one confidence score.

\subsection{Safety-equivalence compression and probes}
Worlds can differ in irrelevant details while agreeing on action safety. Define the action signature
\begin{equation}
    \sigma_t(\omega)=\big(\mathbf{1}[a\in\Gamma_t(\omega)]\big)_{a\in\mathcal{A}_t}.
\end{equation}
Worlds with the same signature are safety-equivalent. Replacing each equivalence class by one representative preserves the certified action set exactly, because certification depends only on whether every retained signature marks an action safe.

When no action is certified, a probe $p\in\mathcal{P}_t$ produces outcome $y$ with cost $c(p)$ and updates the world set to $\Wset_{t+1}^{y,p}$. The controller seeks probes that reduce the worlds blocking certification, while keeping side effects and interaction cost low. This commitprobefallback decision is developed next.

\section{The \textsc{SafeCommit} Approach}
\label{sec:method}

\textsc{SafeCommit} is placed between an agent's memory/tool reasoning loop and external execution. It does not replace the base agent, access control, or sandboxing. Instead, it decides whether the evidence is sufficient to release a proposed action.

\begin{figure}[H]
\centering
\begin{tikzpicture}[
    font=\small,
    node distance=4mm and 4.5mm,
    box/.style={
        draw,
        rounded corners=2pt,
        align=center,
        minimum height=8mm,
        text width=0.27\linewidth,
        inner sep=4pt
    },
    wide/.style={
        draw,
        rounded corners=2pt,
        align=center,
        minimum height=8mm,
        text width=0.40\linewidth,
        inner sep=4pt
    },
    decision/.style={
        draw,
        diamond,
        aspect=2.2,
        align=center,
        inner sep=1.5pt,
        text width=0.20\linewidth
    },
    arr/.style={
        -{Latex[length=2mm]},
        thick
    },
    loop/.style={
        -{Latex[length=2mm]},
        thick,
        dashed,
        rounded corners=2pt
    }
]

% Top row
\node[wide,fill=blue!6] (evidence) {
    \textbf{1. Gather evidence}\\
    memory, observations, tool outputs, provenance, policy
};

\node[wide,fill=orange!8,right=8mm of evidence] (worlds) {
    \textbf{2. Construct plausible worlds}\\
    hard filtering $+$ calibrated nonconformity threshold
};

% Decision gate
\node[
    decision,
    fill=green!8,
    below=7mm of $(evidence)!0.5!(worlds)$
] (gate) {
    Safe in every retained world?
};

% Bottom row
\node[
    box,
    fill=green!10,
    below left=7mm and 12mm of gate
] (commit) {
    \textbf{3a. Commit}\\
    release the certified action
};

\node[
    box,
    fill=yellow!12,
    below=7mm of gate
] (probe) {
    \textbf{3b. Probe}\\
    clarify, inspect metadata, check permission, stage, or simulate
};

\node[
    box,
    fill=red!7,
    below right=7mm and 12mm of gate
] (fallback) {
    \textbf{3c. Fallback}\\
    defer, escalate, or abstain
};

% Main arrows
\draw[arr] (evidence) -- (worlds);
\draw[arr] (worlds) -- (gate);

\draw[arr]
    (gate) --
    node[left,font=\scriptsize]{yes}
    (commit);

\draw[arr]
    (gate) --
    node[right,font=\scriptsize]{no, useful probe}
    (probe);

\draw[arr]
    (gate) --
    node[right,font=\scriptsize]{no budget/evidence}
    (fallback);

% Feedback-loop coordinates:
% 1. Move below the probe box.
% 2. Move left of the commit box.
% 3. Move upward outside the commit box.
% 4. Enter the evidence box from its left side.
\coordinate (belowprobe) at ([yshift=-8mm]probe.south);
\coordinate (leftofcommit) at ($(commit.west)+(-10mm,0)$);
\coordinate (lowerleft) at (leftofcommit |- belowprobe);
\coordinate (upperleft) at (leftofcommit |- evidence.west);

\draw[loop]
    (probe.south)
    -- (belowprobe)
    -- (lowerleft)
    -- (upperleft)
    -- (evidence.west);

\end{tikzpicture}

\caption{
\textsc{SafeCommit} converts memory uncertainty into an explicit
commit-probe-fallback decision. A proposed action is released only
when no retained plausible world makes it unsafe; otherwise, a
low-side-effect probe updates the evidence and repeats the check.
}
\label{fig:approach}
\end{figure}
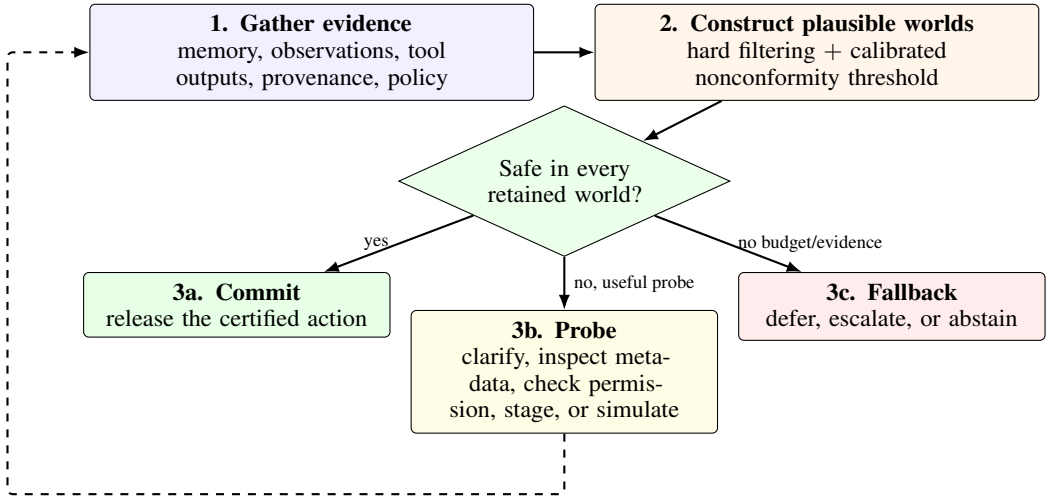

\subsection{Plausible-world construction}
The default constructor has four steps. First, it proposes a bounded support from retrieved memories, current observations, recent tool results, provenance fields, policy constraints, and domain rules. Second, it removes worlds that contradict hard evidence, such as impossible timestamps, invalid permissions, or mutually inconsistent tool outputs. Third, it applies Eq.~\eqref{eq:worldset} using a calibration-only threshold. Fourth, it merges worlds with identical action signatures. The LLM may help propose candidate worlds, but the retained set is constrained by structured evidence and policy predicates rather than by the model rationale alone.

The safety map $\Gamma_t$ is domain-specific. For a file action, it can encode protected paths, reversibility, and the effect revealed by a staged diff. For email, it can encode recipient authorization, attachment sensitivity, and whether sending is irreversible. For an administrative update, it can encode current permission and approval requirements. This interface makes the framework complementary to access-control and effect-mediation systems: their outputs become evidence or safety predicates used by the certificate.

\subsection{A concrete controller interface}
A practical implementation can keep the base agent and controller loosely coupled. The base agent returns a structured proposal containing candidate actions, candidate probes, and the evidence identifiers used in its rationale. An evidence adapter normalizes each item into a record with a proposition, source, timestamp, entity, tool status, and policy role. The world constructor expands unresolved conflicts into alternative assignments, while a domain verifier implements $\Gamma_t(\omega)$. This separation prevents a fluent rationale from serving as its own safety proof.

One simple nonconformity decomposition is
\begin{equation}
\begin{split}
 s_t(\omega)={}&\lambda_{\mathrm{prov}}\psi_{\mathrm{prov}}(\omega)
 +\lambda_{\mathrm{cons}}\psi_{\mathrm{cons}}(\omega)\\
 &+\lambda_{\mathrm{tool}}\psi_{\mathrm{tool}}(\omega)
 +\lambda_{\mathrm{policy}}\psi_{\mathrm{policy}}(\omega),
\end{split}
\label{eq:score}
\end{equation}
where the terms measure provenance or age, cross-evidence inconsistency, disagreement with current tool outputs, and policy incompatibility. The score can be rule-based, learned, or hybrid; the risk statement depends on calibrated coverage, not on a specific score family. Hard policy violations should be filtered before calibration rather than compensated by a low score elsewhere.

Calibration uses histories with known latent worlds or audited safety-relevant labels. For each calibration case, the constructor scores the true world, and $\kappa_\alpha$ is the standard finite-sample quantile. Prompt design, score weights, and probe costs are selected on separate development data. This separation is important: choosing $\kappa_\alpha$ on the test set would destroy the stated risk interpretation. In deployment, calibration should be refreshed when tool schemas, memory policies, or task distributions change, and coverage should be monitored by domain and action severity.

The controller should also log enough information to audit a decision: the retained safety-equivalence classes, the action signatures, the certificate or blocking worlds, the selected probe and its expected shrinkage, and the final fallback reason. These records are substantially smaller than full chain-of-thought traces and correspond directly to the safety decision. They also make it possible to diagnose whether a failure came from a missing world, a bad score, an incorrect safety map, or an unhelpful probe.

\subsection{Certificate gate and action choice}
For every candidate action, the controller computes Eq.~\eqref{eq:uncertified}. Only actions with an empty uncertified region enter the eligible set
\begin{equation}
    \mathcal{A}_t^{\mathrm{cert}}=
    \{a\in\mathcal{A}_t:\Ucal_t(a)=\emptyset\}.
\end{equation}
If several actions are certified, the controller may maximize task reward minus a side-effect severity penalty,
\begin{equation}
 a_t^\star\in\arg\max_{a\in\mathcal{A}_t^{\mathrm{cert}}}
 \big[R_t(a)-\lambda_{\mathrm{sev}}\,\mathrm{sev}_t(a)\big].
\end{equation}
This tie-breaking changes utility but not certificate validity because every eligible action is safe across the retained set.

\subsection{Certificate-shrinking probes}
If $\mathcal{A}_t^{\mathrm{cert}}$ is empty, the controller evaluates low-side-effect probes. Let $\mu(\cdot)$ be a monotone mass over worlds or safety-equivalence classes. For outcome $y$, define the post-probe uncertified region $\Ucal_{t+1}^{y,p}(a)$. The expected certificate shrinkage is
\begin{equation}
\Delta_t(a,p)=\E_y\!\left[\mu(\Ucal_t(a))-\mu(\Ucal_{t+1}^{y,p}(a))\right],
\end{equation}
and the default score is
\begin{equation}
    \Phi_t(p)=\frac{\max_{a\in\mathcal{A}_t}\Delta_t(a,p)}{c(p)+\epsilon}.
    \label{eq:probe}
\end{equation}
Unlike generic information gain, Eq.~\eqref{eq:probe} rewards information only when it removes worlds that block an action certificate. A file hash or staged diff may therefore outrank a broad clarification if it directly distinguishes a harmless write from a destructive one.

\begin{algorithm}[t]
\caption{Certified commitment under memory uncertainty}
\label{alg:safecommit}
\begin{algorithmic}[1]
\Require history $h_t$, actions $\mathcal{A}_t$, probes $\mathcal{P}_t$, risk $\alpha$, threshold $\eta$, budget $B$
\While{budget remains}
    \State construct and compress $\Wset_t=C_\alpha(h_t)$
    \State compute $\Ucal_t(a)$ for all $a\in\mathcal{A}_t$
    \If{some $a$ has $\Ucal_t(a)=\emptyset$}
        \State \Return \textsc{commit}$(a_t^\star)$
    \EndIf
    \State $p_t^\star\gets\arg\max_{p\in\mathcal{P}_t}\Phi_t(p)$
    \If{$\Phi_t(p_t^\star)<\eta$}
        \State \Return \textsc{defer}, \textsc{escalate}, or \textsc{abstain}
    \EndIf
    \State execute $p_t^\star$, observe $y$, and update $h_t$
\EndWhile
\State \Return conservative fallback
\end{algorithmic}
\end{algorithm}

\subsection{Effect-aware probing, fallback, and cost}
Probes can be observational, such as a metadata read or permission check, or effect-revealing, such as a dry run, staged diff, or sandboxed execution. Reversible mutations may be staged and inspected before final commitment. Irreversible actions, including sending a message, exposing a secret, or executing a nonrecoverable write, must be certified before execution.

The loop stops when an action is certified, the probe value falls below $\eta$, or a probe/time budget is exhausted. \textsc{defer} is appropriate when evidence may arrive later, \textsc{escalate} when external approval is required, and \textsc{abstain} when no safe action appears achievable. If $K_t$ safety-equivalence classes and $A_t$ actions remain, certificate evaluation costs $O(K_tA_t)$. Exact scoring of probes with $Y_p$ possible outcomes costs $O(\sum_p A_tK_tY_p)$; bounded supports and sampled outcomes make this practical for small decision-time sets.

\section{Reproducible Proof-of-Concept Evaluation}
\label{sec:evaluation}

The purpose of this evaluation is to test the mechanism in a setting where the latent world, safety map, and probe outcomes are exactly observable. It is not presented as evidence that the same magnitudes will hold for deployed LLM agents. Instead, it provides an executable reference point that can later be replaced with richer world constructors, tool environments, and LLM proposal policies.

\subsection{Controlled benchmark and protocol}
The artifact includes \textsc{SafeCommitBench-Controlled}, a frozen JSONL benchmark with train, validation, calibration, and test splits, plus the generator that created them. Each episode belongs to one of four memory-uncertainty families: stale records, conflicting records, poisoned records, or authorization drift. An episode contains three to six candidate worlds, one latent true world, three possible actions, a nonconformity score for every world, a finite proposed support, and deterministic binary metadata-style probes. Distractor worlds can support a different action and receive deceptively strong evidence according to the scenario family. In the main experiment, the true world is always included in the proposed support so that the experiment isolates conformal filtering and the commitprobefallback policy. Appendix~\ref{app:proofs}, \ref{app:implementation} and  \ref{app:representation_experiment} specifies the schema and studies proposal misses separately.

We use 10 independent seeds. Per seed, 2,000 episodes estimate the split-conformal threshold and 4,000 disjoint episodes are evaluated. The main setting uses $\alpha=0.05$ and at most two probes. We compare: \textbf{Single-world}, which acts on the lowest-nonconformity world; \textbf{Generic one-probe}, which checks a random retained world and then acts on the best remaining hypothesis; \textbf{Conflict abstention}, which falls back on close score conflicts; \textbf{SafeCommit (no probe)}, which certifies or falls back; and full \textbf{SafeCommit}, which selects probes by expected certificate shrinkage. All methods receive identical episodes and proposed support.

Metrics are unsafe commit rate (UCR), task success (TS), commit coverage (CC), fallback rate (FR), and mean probes. UCR is the fraction of all episodes ending in an unsafe action, matching the unconditional statement in Proposition~\ref{prop:risk}.

\subsection{Aggregate results}
\begin{table}[t]
\centering
\small
\caption{Controlled results at $\alpha=0.05$. UCR, TS, CC, and FR are percentages; lower is better for UCR, FR, and probes. Values are means over 10 seeds.}
\label{tab:main_results}
\begin{tabular}{lccccc}
\toprule
\textbf{Method} & \textbf{UCR$\downarrow$} & \textbf{TS$\uparrow$} & \textbf{CC$\uparrow$} & \textbf{FR$\downarrow$} & \textbf{Probes$\downarrow$} \\
\midrule
Single-world & 41.2 & 58.8 & 100.0 & 0.0 & 0.00 \\
Generic one-probe & 4.8 & 95.2 & 100.0 & 0.0 & 0.64 \\
Conflict abstention & 19.2 & 48.1 & 67.3 & 32.7 & 0.00 \\
\textsc{SafeCommit} (no probe) & 2.5 & 44.7 & 47.2 & 52.8 & 0.00 \\
\textbf{\textsc{SafeCommit}} & \textbf{2.6} & \textbf{97.4} & \textbf{100.0} & \textbf{0.0} & \textbf{0.55} \\
\bottomrule
\end{tabular}
\end{table}

Table~\ref{tab:main_results} illustrates three intended behaviors. First, acting on one plausible world is unsafe when memory can be misleading: Single-world commits unsafely in 41.2\% of episodes. Second, certification without probing is safe but conservative: \textsc{SafeCommit} (no probe) lowers UCR to 2.5\% but completes only 44.7\% of tasks. Third, targeted probes recover utility: full \textsc{SafeCommit} reaches 97.4\% task success with 2.6\% UCR and 0.55 probes per episode. Relative to Generic one-probe, targeted certification approximately halves unsafe commits while improving task success.

\subsection{Behavior across memory failures}
\begin{table*}[t]
\centering
\small
\setlength{\tabcolsep}{4.5pt}
\caption{Family-wise controlled results at $\alpha=0.05$. Each family is evaluated on 4,000 disjoint episodes per seed. Values are percentages averaged over 10 seeds.}
\label{tab:family_results}
\begin{tabular}{lcccccccc}
\toprule
 & \multicolumn{2}{c}{Stale} & \multicolumn{2}{c}{Conflict} & \multicolumn{2}{c}{Poisoned} & \multicolumn{2}{c}{Auth. drift} \\
\cmidrule(lr){2-3}\cmidrule(lr){4-5}\cmidrule(lr){6-7}\cmidrule(lr){8-9}
\textbf{Method} & UCR$\downarrow$ & TS$\uparrow$ & UCR$\downarrow$ & TS$\uparrow$ & UCR$\downarrow$ & TS$\uparrow$ & UCR$\downarrow$ & TS$\uparrow$ \\
\midrule
Single-world & 24.1 & 75.9 & 25.5 & 74.5 & 62.5 & 37.5 & 51.3 & 48.7 \\
Generic one-probe & 1.6 & 98.4 & 8.2 & 91.8 & 4.8 & 95.2 & 4.2 & 95.7 \\
\textbf{\textsc{SafeCommit}} & \textbf{1.2} & \textbf{98.8} & \textbf{1.2} & \textbf{98.8} & \textbf{3.9} & \textbf{96.1} & \textbf{3.5} & \textbf{96.5} \\
\bottomrule
\end{tabular}
\end{table*}

Table~\ref{tab:family_results} is the additional main-paper experiment. Single-world inference is particularly vulnerable to poisoned memory and authorization drift, where misleading evidence can appear stronger than the current state. A random probe improves all families but remains weak on conflicts because it can inspect a world unrelated to the action disagreement. \textsc{SafeCommit} keeps UCR between 1.2\% and 3.9\% across the four families while maintaining at least 96.1\% task success. The poisoned and authorization-shifted families remain the hardest because the true world receives a modest nonconformity penalty and deceptive alternatives are frequently ranked highly. This breakdown supports the intended role of provenance- and policy-aware world construction rather than claiming that the controlled family generator captures every real attack.

\subsection{Probe-budget ablation}
\begin{table}[t]
\centering
\small
\caption{Effect of the maximum probe budget at $\alpha=0.05$. A small active budget recovers most of the utility lost by certification-only fallback.}
\label{tab:probe_budget}
\begin{tabular}{ccccc}
\toprule
Probe budget & UCR$\downarrow$ & TS$\uparrow$ & CC$\uparrow$ & FR$\downarrow$ \\
\midrule
0 & 2.5\% & 44.7\% & 47.2\% & 52.8\% \\
1 & 2.6\% & 95.1\% & 97.6\% & 2.4\% \\
2 & 2.6\% & 97.4\% & 100.0\% & 0.0\% \\
4 & 2.6\% & 97.4\% & 100.0\% & 0.0\% \\
\bottomrule
\end{tabular}
\end{table}

The probe-budget ablation isolates the contribution of active evidence acquisition. With no probes, the certificate gate keeps UCR low but falls back on 52.8\% of episodes. One targeted probe raises task success from 44.7\% to 95.1\% and reduces fallback to 2.4\%, without materially increasing UCR. A second probe resolves the remaining multi-world conflicts, while increasing the budget from two to four produces no measurable benefit in this bounded benchmark. This saturation is specific to episodes containing only three to six worlds and deterministic binary probes, but it shows why certification and probing should be evaluated together: a certificate-only method may look safe mainly because it refuses to act, whereas a probe policy can recover coverage without relaxing the commit rule.

\subsection{Risk and probe tradeoff}
\begin{table}[t]
\centering
\small
\caption{Risk-level sweep for full \textsc{SafeCommit}. The calibrated target controls the retained set and changes both unsafe commitment and probe demand.}
\label{tab:risk_sweep}
\begin{tabular}{ccccc}
\toprule
Target $\alpha$ & UCR & Task success & Commit coverage & Mean probes \\
\midrule
1\% & 0.6\% & 99.4\% & 99.9\% & 0.68 \\
5\% & 2.6\% & 97.4\% & 100.0\% & 0.55 \\
10\% & 4.8\% & 95.2\% & 100.0\% & 0.50 \\
\bottomrule
\end{tabular}
\end{table}

At target levels $\alpha\in\{0.01,0.05,0.10\}$, full \textsc{SafeCommit} obtains UCRs of 0.6\%, 2.6\%, and 4.8\%, respectively (Table~\ref{tab:risk_sweep}). These values remain below the corresponding target levels because excluding the true world does not always produce an unsafe committed action. Increasing $\alpha$ retains fewer worlds, which reduces the average probe count from 0.68 to 0.50 but increases unsafe commitment. Thus, the risk level changes a measurable safetyinteraction frontier rather than acting only as a nominal confidence parameter.

\subsection{Scope}
The benchmark uses explicit latent worlds, deterministic probe outcomes, and hand-specified safety maps. Real systems must construct worlds from noisy unstructured evidence, estimate probe outcomes, and encode domain-specific safety semantics. Distribution shift can violate conformal exchangeability, and missing a safety-relevant world adds the representation term $\beta$ in Eq.~\eqref{eq:representation}. The results should therefore be read as a mechanism check and reproducible starting point. A complete empirical study would require realistic tool environments, audited labels, diverse LLM backends, and evaluation of world-proposal quality, latency, and adversarial adaptation.

\section{Design Implications and Open Questions}
\label{sec:discussion}

\paragraph{Decision-relevant completeness, not exhaustive world modeling.}
A common concern is that a plausible-world method must enumerate an intractable environment state. The certificate only requires distinctions that change whether a candidate action is safe. For a file deletion, file ownership, path resolution, protection status, and reversibility may matter, while unrelated file contents do not. Safety-equivalence compression formalizes this principle: two worlds can be merged whenever they induce the same action-safety signature. The practical objective for a world constructor is therefore not full factual reconstruction, but high recall over safety-relevant alternatives. Auditing should measure which unsafe signatures were omitted, not only whether a natural-language world description matched a reference answer.

\paragraph{The safety map is a first-class interface.}
The map $\Gamma_t(\omega)$ should be implemented independently of the model proposing actions. It may combine deterministic policy checks, typed preconditions, human approval requirements, effect simulators, and learned verifiers. This separation makes failure analysis clearer. If the true world was retained but an unsafe action was certified, the safety map is wrong. If a necessary world was absent, the constructor failed. If the correct worlds were retained but the system fell back unnecessarily, the probe policy or budget was inadequate. Treating these as distinct interfaces is more actionable than attributing every failure to general model uncertainty.

\paragraph{Sequential risk needs additional care.}
Proposition~\ref{prop:risk} applies to a commitment decision whose plausible-world set satisfies the coverage condition. A long-horizon agent may make many commitments, adapt its memory after each action, and select when to stop based on prior outcomes. Per-step calibration does not automatically imply a tight mission-level guarantee. Conservative deployments can allocate a total risk budget across commitments, use time-uniform conformal methods, or recertify after every state-changing action. Developing useful guarantees that account for adaptive probing, repeated commits, and correlated world-construction errors is an important extension.

\paragraph{Calibration must follow operational change.}
The exchangeability assumption can fail when tools, policies, users, or attack strategies change. A deployment should therefore monitor empirical coverage proxies and stratify calibration by action family, severity, and source type. A threshold calibrated for reversible file edits should not automatically govern secret access or financial transactions. When labels are delayed, staged execution and human review can provide conservative evidence for recalibration. The representation term $\beta$ is equally important: better calibration cannot repair a world that was never proposed.

\paragraph{Fallback is part of useful behavior.}
A safe system is not one that always abstains. \textsc{SafeCommit} makes the tradeoff visible through commit coverage, task success, probe cost, and fallback rate. The type of fallback should match the missing evidence. Deferral is appropriate when a fresh observation will arrive; escalation is appropriate when authority rather than information is missing; abstention is appropriate when no acceptable action exists. This distinction can also guide probe design: a permission check targets potential escalation, while a staged diff targets uncertainty about effects.

\paragraph{Research opportunities.}
Several questions remain open: how to generate compact worlds from unstructured memory without omitting rare hazards; how to learn probe-outcome models while preserving conservative guarantees; how to compose certificates across multi-agent plans; and how to evaluate representation error when the true latent world is only partially observable. Realistic studies should also compare decision-level certificates with scalar confidence, policy-only gating, and execution-only containment under matched tools and action budgets. The included simulator is intended to make these extensions easy to prototype while keeping the central commitment rule explicit.

\section{Discussion and Conclusion}
\label{sec:conclusion}

\textsc{SafeCommit} treats external action as a release decision rather than the automatic endpoint of agent reasoning. Its central rule is simple: commit only when the proposed action is safe across a calibrated set of plausible worlds; otherwise gather evidence targeted at the worlds blocking certification or return a conservative fallback. This separates plausible justification from certifiable action and gives unsafe commitment an explicit risk interpretation.

The framework is deliberately modular. Memory systems propose relevant state, tools expose current evidence, access-control systems contribute policy constraints, and staging or sandboxing mechanisms provide effect-revealing probes. \textsc{SafeCommit} connects these components at decision time. Its guarantee is only as strong as world coverage and the correctness of the safety map, so it should complement rather than replace memory security, authorization, human oversight, and execution containment.

The immediate contribution is an approach and a reproducible reference implementation. The next step is to test whether realistic world constructors and probe libraries can preserve the same safetyutility behavior in long-horizon computer-use, file, communication, and administrative agents.

\bibliographystyle{plainnat}
\bibliography{ref}

\clearpage
\appendix
\section*{Appendix}

\section{Proofs and Additional Concepts}
\label{app:proofs}

\subsection{Finite-sample calibration threshold}
Let the calibration set be $\mathcal{H}_{\mathrm{cal}}=\{(h_i,\omega_i^\star)\}_{i=1}^{n_{\mathrm{cal}}}$ and let $q_i=s_i(\omega_i^\star)$ be the nonconformity score assigned to the true world. For target risk $\alpha$, the implementation uses the split-conformal order statistic
\begin{equation}
\kappa_\alpha=q_{(r)},\qquad
r=\min\!\left\{n_{\mathrm{cal}},\left\lceil(n_{\mathrm{cal}}+1)(1-\alpha)\right\rceil\right\},
\label{eq:app_quantile}
\end{equation}
where $q_{(r)}$ is the $r$th smallest calibration score. Under exchangeability between calibration and test cases, Eq.~\eqref{eq:app_quantile} gives the marginal world-set coverage in Eq.~\eqref{eq:coverage}. Development choices such as prompt wording, score weights, probe costs, and support caps must be fixed before the calibration split is used.

\subsection{Full unsafe-commit proof}
\begin{proposition}[Unsafe-commit control, restated]
Assume $\Prob(\omega_\tau^\star\in\Wset_\tau)\ge 1-\alpha$ at the stopping time and that the controller commits only when $\Ucal_\tau(a_\tau)=\emptyset$. Then $\Prob(a_\tau\notin\Gamma_\tau(\omega_\tau^\star))\le\alpha$.
\end{proposition}
\begin{proof}
The certificate condition implies
\begin{equation}
\Ucal_\tau(a_\tau)=\emptyset
\quad\Longrightarrow\quad
\forall\omega\in\Wset_\tau,\;a_\tau\in\Gamma_\tau(\omega).
\end{equation}
Consequently, whenever the true world belongs to the retained set, the committed action is safe in that true world. Therefore the unsafe event is contained in the coverage-failure event:
\begin{equation}
\{a_\tau\notin\Gamma_\tau(\omega_\tau^\star)\}
\subseteq
\{\omega_\tau^\star\notin\Wset_\tau\}.
\end{equation}
Taking probabilities and applying Eq.~\eqref{eq:coverage} completes the proof.
\end{proof}

\subsection{Exactness of safety-equivalence compression}
Define the certified action set over retained worlds as
\begin{equation}
\mathcal{C}(\Wset_t)=\{a\in\mathcal{A}_t:\forall\omega\in\Wset_t,\;a\in\Gamma_t(\omega)\}.
\end{equation}
Let $\omega\sim_t\omega'$ when the two worlds have the same action signature $\sigma_t(\omega)=\sigma_t(\omega')$, and let $\overline{\Wset}_t=\Wset_t/{\sim_t}$.

\begin{proposition}[Compression preserves certification]
For any retained set, $\mathcal{C}(\Wset_t)=\mathcal{C}(\overline{\Wset}_t)$ when each equivalence class is represented by any one of its members.
\end{proposition}
\begin{proof}
Every member of an equivalence class has the same binary safety judgment for every candidate action. Replacing the class with one representative therefore leaves unchanged whether an action is marked safe by all retained worlds. The universal condition defining $\mathcal{C}$ is identical before and after replacement.
\end{proof}
This result motivates storing safety signatures rather than verbose latent descriptions once the world constructor has finished. Compression is exact for certification even if worlds differ in facts irrelevant to the current action set.

\subsection{Representation error}
Let $M_t=\{\omega_t^\star\notin\widehat\Omega_t(h_t)\}$ be the event that the proposal step misses the true world, with $\Prob(M_t)\le\beta$. Suppose that conditional on proposing the true world, calibration retains it with probability at least $1-\alpha$. Then
\begin{align}
\Prob(\omega_t^\star\notin\Wset_t)
&=\Prob(M_t)+\Prob(\omega_t^\star\notin\Wset_t,M_t^c)\\
&\le\beta+(1-\beta)\alpha\le\alpha+\beta.
\end{align}
Combining this result with the event inclusion in the previous proof yields Eq.~\eqref{eq:representation}. The decomposition is operationally useful: calibration data can estimate the $\alpha$ component, whereas support audits and oracle-world studies are needed to estimate or upper-bound $\beta$.

\subsection{Certificate variants}
The main paper uses a hard certificate because its safety interpretation is direct. Two extensions are possible. A \emph{soft certificate} permits residual weighted unsafe mass below a tolerance $\rho$,
\begin{equation}
\sum_{\omega\in\Wset_t}w_t(\omega)\mathbf{1}[a\notin\Gamma_t(\omega)]\le\rho,
\end{equation}
trading a weaker guarantee for higher commit coverage. A \emph{robust certificate} additionally requires the action signature to remain safe under a specified perturbation set over memory atoms, tool outputs, or support construction. These variants are not used in the reported controlled results.

\section{Implementation Details}
\label{app:implementation}

\subsection{World construction and evidence normalization}
A practical controller can normalize evidence into typed records with fields for proposition, entity, source, timestamp, tool status, and policy role. Unresolved contradictions become alternative assignments in the finite support. Hard inconsistencies and explicit policy violations are removed before scoring. The default score in Eq.~\eqref{eq:score} combines provenance, cross-evidence consistency, current-tool agreement, and policy compatibility. The simulator uses a scalar score sampled by family, but the certificate and calibration code do not depend on that particular generator.

\subsection{Probe outcome model}
For exact scoring, a probe $p$ has an outcome model $\widehat P_t(y\mid p)$ induced by the retained worlds. With normalized weights $w_t(\omega)$,
\begin{equation}
\widehat P_t(y\mid p)=\sum_{\omega\in\Wset_t}w_t(\omega)P(y\mid\omega,p).
\end{equation}
The released simulator uses deterministic identity-style probes: a probe asks whether one retained interpretation matches the hidden state. The full controller then chooses the probe maximizing expected reduction in minimum unsafe mass per unit cost. In a real system, analogous probes can be metadata reads, permission checks, dry runs, staged diffs, or sandboxed executions.

\subsection{Failure attribution}
The architecture separates four failure sources. A \emph{proposal failure} omits an important world. A \emph{calibration failure} scores the true proposed world above the retained threshold. A \emph{safety-map failure} marks an unsafe action safe in a retained world. A \emph{probe-policy failure} retains the correct worlds but spends budget on evidence that does not resolve the blocking action disagreement. The release artifact logs enough fields to diagnose the first, second, and fourth sources in the controlled setting; a deployed system must additionally audit its domain safety map.

\section{Additional Experiment: Representation Misses}
\label{app:representation_experiment}

The main evaluation sets $\beta=0$ to isolate conformal filtering and probing. This appendix experiment tests the separate representation term by omitting the true world from the proposed support with probability $\beta\in\{0,0.02,0.05,0.10\}$. Calibration remains unchanged and uses complete proposal support. The controller cannot recover an omitted world through calibration; probes can only distinguish among proposed alternatives. We use the same 10 seeds, $\alpha=0.05$, 2,000 calibration episodes per seed, and 4,000 test episodes per seed.

\begin{table}[h]
\centering
\small
\caption{Sensitivity to world-proposal misses. The empirical UCR remains below the loose $\alpha+\beta$ bound, while utility degrades as omitted worlds force incorrect certification.}
\label{tab:representation_sweep}
\begin{tabular}{ccccc}
\toprule
Observed $\beta$ & UCR$\downarrow$ & Bound $\alpha+\beta$ & TS$\uparrow$ & FR$\downarrow$ \\
\midrule
0.0\% & 2.6\% & 5.0\% & 97.4\% & 0.0\% \\
2.0\% & 3.9\% & 7.0\% & 96.1\% & 0.0\% \\
5.0\% & 6.2\% & 10.0\% & 93.8\% & 0.0\% \\
10.1\% & 10.1\% & 15.1\% & 89.9\% & 0.0\% \\
\bottomrule
\end{tabular}
\end{table}

Table~\ref{tab:representation_sweep} shows the expected separation. At observed $\beta=0$, UCR is 2.6\%, below the 5\% target. As the miss rate increases to approximately 10\%, UCR rises to 10.1\%, while the conservative bound rises to 15.1\%. Task success falls because the omitted true world can leave a wrong action unanimously supported by the remaining set. The experiment does not prove that real world constructors will have low $\beta$; rather, it demonstrates why proposal recall must be evaluated independently of calibration quality.

\section{Discussion: Toward Trustworthy and Adaptive Commitment}
\label{sec:discussion_trustworthy_commitment}

SAFECOMMIT provides a principled gate for deciding when memory-grounded agents may execute side-effectful actions under unresolved uncertainty. Its calibrated world sets and targeted probes could be strengthened by personalized feedback mechanisms that adapt clarification and escalation to specific users or operators~\cite{ranjan2026persa}. Privacy auditing is also important because persistent agent memory may retain sensitive information; gradient-induced membership analysis could help detect such exposure~\cite{ranjan2026g}. When stale, poisoned, or unsafe knowledge is identified, targeted unlearning approaches may remove the responsible representations while preserving useful agent capabilities~\cite{ranjan2026razor,ranjan2026vla}. Retrieval-grounded debiasing and structured mitigation can further improve world construction by reducing unsupported or biased evidence~\cite{ranjan2026catrag,ranjan2026position}. Together with explainability, domain trustworthiness, and edge-deployment safeguards~\cite{ranjan2026listening,kumar2024trustworthiness,grover2026embodied}, these directions could extend SAFECOMMIT from controlled certification toward auditable, adaptive, and deployment-ready agent safety.

\section{Limitations and Intended Use}
\label{app:limitations}

The controlled benchmark makes several strong assumptions: safety labels are exact, probes are deterministic, world supports are small, and the action set is fixed before certification. The conformal statement is marginal and assumes exchangeability between calibration and deployment cases. Repeated adaptive commitments require a mission-level risk allocation or a time-uniform treatment. The method also requires a domain safety map that is independent of the LLM proposal path; a wrong safety map can certify a harmful action even when the true world is retained.

Accordingly, the released code should be treated as a reference implementation and benchmark scaffold. It is suitable for testing alternative scoring rules, proposal mechanisms, probe policies, and support-compression schemes. It is not a production safety layer, does not replace authorization or sandboxing, and should not be used to justify high-impact actions without domain-specific validation and oversight.

\end{document}